\newif\iflong
\newif\ifshort

 \longtrue

\iflong
\else
\shorttrue
\fi

\documentclass[nohyperref]{article}

 \usepackage{microtype}
\usepackage{graphicx}
\usepackage{wrapfig}
\usepackage{subfigure,xspace}
\usepackage{tikz}
\usepackage{natbib}
\usetikzlibrary{shapes, positioning, patterns,decorations.pathreplacing}\tikzstyle{vertex}=[circle, draw, inner sep=0pt, minimum size=4pt,outer sep = 1pt]

\usepackage{booktabs}  
    \usepackage{hyperref}

 \usepackage[final]{neurips_2023}

 \usepackage{amsmath}
\usepackage{amssymb}
\usepackage{mathtools}
\usepackage{amsthm}

 \usepackage[capitalize,noabbrev]{cleveref}

   \theoremstyle{plain}

   \usepackage[textsize=tiny]{todonotes}

 \usepackage{complexity}
 \usepackage{enumitem}

\usepackage[utf8]{inputenc}  \usepackage{hyperref}        \usepackage{cleveref}

\usepackage{url}             \usepackage{booktabs}        \usepackage{amsfonts}        \usepackage{nicefrac}        \usepackage{microtype}       \usepackage{xcolor}          \usepackage{boxedminipage}

\renewcommand{\cc}[1]{{\mbox{\textnormal{\textsf{#1}}}}\xspace}   

\newcommand{\bigoh}{\mathcal{O}}

\newcommand{\RR}{\mathbb{R}}

 \newcommand{\ER}{\ensuremath \exists \cc{R}}

\renewcommand{\poly}{\mathrm{poly}}

\newcommand{\NNT}{\textsc{NNT}\xspace}
\newcommand{\lNNT}{\textsc{Lin-NNT}\xspace}
\newcommand{\rNNT}{\textsc{ReLU-NNT}\xspace}
\newcommand{\mNNT}{\textsc{Mixed-NNT}\xspace}

\newtheorem{theorem}{Theorem}

\newtheorem{observation}[theorem]{Observation}
\newtheorem{proposition}[theorem]{Proposition}

\newtheorem{corollary}[theorem]{Corollary}
\newtheorem{lemma}[theorem]{Lemma}
\newtheorem{lemmabis}{Lemma}[theorem]

\author{   Cornelius Brand \\
  Algorithms \& Complexity Group \\
  Vienna University of Technology \\
  Favoritenstraße 9-11, 1040 Vienna, Austria \\
  \texttt{cbrand@ac.tuwien.ac.at} \\
  \And
  Robert Ganian \\
  Algorithms \& Complexity Group \\
  Vienna University of Technology \\
  Favoritenstraße 9-11, 1040 Vienna, Austria \\
  \texttt{rganian@gmail.com} \\
  \And
  Mathis Rocton \\
  Algorithms \& Complexity Group \\
  Vienna University of Technology \\
  Favoritenstraße 9-11, 1040 Vienna, Austria \\
  \texttt{mrocton@ac.tuwien.ac.at} \\
                                                               }

\title{New Complexity-Theoretic Frontiers of Tractability for Neural Network Training (Full Version)}

\begin{document}

\maketitle

\begin{abstract}
 In spite of the fundamental role of neural networks in contemporary machine learning research, our understanding of the computational complexity of optimally training neural networks remains incomplete even when dealing with the simplest kinds of activation functions. Indeed, while there has been a number of very recent results that establish ever-tighter lower bounds for the problem under linear and ReLU activation functions, less progress has been made towards the identification of novel polynomial-time tractable network architectures. 
In this article we obtain novel algorithmic upper bounds for training linear- and ReLU-activated neural networks to optimality which push the boundaries of tractability for these problems beyond the previous state of the art. 
In particular, for ReLU networks we establish the polynomial-time tractability of all architectures where hidden neurons have an out-degree of $1$, improving upon the previous algorithm of Arora, Basu, Mianjy and Mukherjee. On the other hand, for networks with linear activation functions we identify the first non-trivial polynomial-time solvable class of networks by obtaining an algorithm that can optimally train network architectures satisfying a novel data throughput condition.
\end{abstract}

\section{Introduction}
\begin{wrapfigure}{r}{0.4\textwidth}
\begin{center}
\vspace{-0.72cm}
\includegraphics[scale=0.5]{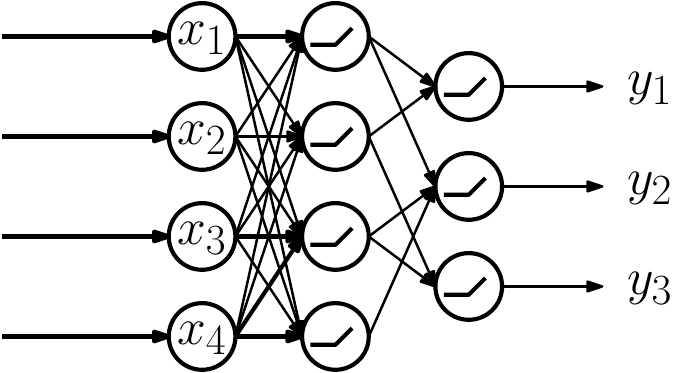}
\end{center}
\caption{A neural network with two layers of hidden, ReLU-activated neurons, computing a function $f: \RR^4 \rightarrow \RR^3$.}
\vspace{-0.6cm}
\label{fig:network}
\end{wrapfigure}

Neural networks are a prominent tool in contemporary machine learning, one which has found ubiquitous applications throughout modern computer science~\citep{DeepL}. A neural network (cf. Figure~\ref{fig:network}) can be thought of as a directed acyclic network consisting of $n$ sources (typically called \emph{input nodes}), and $w$ remaining nodes, which we partition into \emph{output neurons} (the sinks) and \emph{hidden neurons} (all other nodes in the network, typically organized into \emph{layers}).

Given the prominence of neural networks, it is surprising that---in spite of recent efforts---relatively little is known about complexity-theoretic upper bounds for 
 the fundamental problem of Neural Network Training (\NNT): given a network and a training data set containing $m$ samples, compute weights and biases that best fit the training data.\footnote{Formal definitions are provided in the Preliminaries.}
This contrasts recent advances leading to novel algorithmic upper bounds as well as extensive complexity-theoretic landscapes for many other problems arising in machine learning and neural network research, including, e.g., principal component analysis~\citep{SimonovFGP19,DahiyaFPS21}, clustering~\citep{GanianHKOS22}, Bayesian network learning~\citep{OrdyniakS13,GanianK21,GruttemeierK22} and matrix completion~\citep{GanianKOS18}.

Naturally, the complexity of solving \NNT to optimality strongly depends on the activation function used in the network as well as on what kind of restrictions are placed on the structure of the network, but---as we will see---even in the best studied and simplest cases we lack an understanding of the frontiers of polynomial-time tractability. 
The Rectified Linear Unit (ReLU) activation function is a natural first choice to consider for a complexity-theoretic investigation; over the past years, it has become the most popular and widely studied activation function used in neural networks~\citep{GaoCJ20}.
Goel, Klivans, Manurangsi and Reich\-man~(\citeyear{GoelKM021}) recently established that the Neural Network Training problem with ReLU activation functions (\rNNT) is \NP-hard, even when restricted to networks with no hidden neurons, meaning that 
    \rNNT is computationally intractable when no further restrictions are placed on the network. In fact, an even more recent reduction shows that \rNNT is also complete for the complexity class \ER\ when restricted to complete networks~\citep{Fullyconnet} with precisely two input nodes and two output neurons, as opposed to networks with arbitrary structure. Moreover, the problem was shown to be hard for the complexity class \cc{W[1]}~\citep{FroeseHN22} when parameterized by $n$ and also intractable for instances of fixed dimension~\cite{Froese23}.

While these results seem discouraging at first, an immediate question that arises is whether we can at least efficiently train ``small'' networks? In particular, what is the complexity of \rNNT (as well as other variants of \NNT) when restricted to networks where $n$, $w$, or both are (fixed but arbitrary) constants?
 Here we find a stark contrast between the extensive body of research on the complexity of other natural meta-problems\footnote{Consider, e.g., Lenstra's celebrated theorem for solving Integer Linear Programs with constantly-many variables~\citep{Lenstra83} and the multitude of improvements of that initial result~\citep{Kannan87,FrankTardos87,GanianO19,BrandKO21}, or Schaefer's famous dichotomy theorem for Boolean CSPs~\citep{Schaefer78} and its more recent generalization to bounded-domain CSPs~\citep{Bulatov17}.} and how little is known for \textsc{(ReLU-)NNT}.
  Indeed, while recent reductions immediately rule out polynomial-time tractability for networks with constant $w$~\citep{AbrahamsenKM21,FroeseHN22} or constant $n$~\citep{Fullyconnet}, the complexity of \rNNT\ in the base case where both $n$ and $w$ are bounded by a constant (i.e., when training constant-size networks) is a prominent open question in the field. The best partial answer we had so far for this high-profile question lies in the seminal work of Arora, Basu, Mianjy and Mukherjee~(\citeyear{AroraBMM18}) (see also the follow-up work of Boob, Dey and Lan~(\citeyear{BoobDL22})), who developed a polynomial-time algorithm that can solve \rNNT\ for constant-size networks with a single layer of hidden neurons and one output neuron.
 
Naturally, the study of \NNT\ is also relevant for other popular activation functions such as the Sigmoid or Tanh---however, these seem even less amenable to contemporary algorithmic techniques than ReLUs. On the other hand, there exists a class of fundamental activation functions for which the training problem turns out to be easier than for ReLUs: linear functions. Linear-activated networks 
 have been considered in numerous settings as well as theoretical examples~\citep{AbrahamsenKM21,linearact,PanigrahiSG20} and can in many ways be viewed as the ``baseline'' choice for considering the complexity of neural network training. Indeed, unlike for \rNNT, it is not difficult to show that \lNNT (i.e., \NNT with linear activation functions) is polynomial-time tractable for any network of bounded (i.e., constant) size, 
 but is known to be \ER-complete for general networks~\citep{AbrahamsenKM21} (see also Section~\ref{sec:known}). However, the complexity of \lNNT for general networks with bounded $w$---i.e., with constantly many hidden and output neurons---remains open.

\smallskip
\noindent
\textbf{Contributions.} While several papers have recently provided lower bounds for \NNT~\citep{AbrahamsenKM21,FroeseHN22,Fullyconnet}, much less progress has been made towards pushing the frontiers of tractability of the problem. In this paper, we remedy this by providing new algorithmic upper bounds that supersede the previous state of the art for 
\rNNT and \lNNT.

As our first major contribution, we establish the polynomial-time tractability of \rNNT\ for all constant-size networks satisfying the property that all hidden neurons have out-degree at most $1$. We remark that while this condition is trivially satisfied by every neural network that can be handled by the aforementioned algorithms~\citep{AroraBMM18,BoobDL22}, our result may also be applied to much more general networks and in particular can support more than one output neuron and multiple hidden layers. 
The algorithm underlying our result can also deal with networks combining ReLU and linear activation functions. Moreover, its running time matches the recent algorithmic lower bound of~\citet{FroeseHN22} for single-neuron architectures, meaning that it is essentially optimal in that setting.

One remarkable consequence of our exact algorithm for \rNNT\ is that it allows us to formulate a procedure which can effectively deal with \emph{every} constant-size network architecture (even those not satisfying the restriction on the out-degree of hidden neurons). In particular, we show that every constant-size ReLU-activated network architecture can be transformed into a new network architecture which is (1) at least as expressive as the original architecture, and (2) can be trained in polynomial time. Crucially, the depth of the new architecture remains the same and its size depends solely on the size of the original architecture.

Next, we turn to the simpler setting of training neural networks with linear activation functions, i.e., \lNNT.
Recall that here, we are dealing with networks containing a bounded number of hidden/output neurons, but a potentially large number of input nodes. 
Once we depart from the case of architectures with complete connections between consecutive layers (which can be solved by a direct application of linear regression), training neural networks to optimality on a given data set is known to be hopelessly hard~\citep{AbrahamsenKM21} unless the admissible architectures/data are restricted to highly specialized scenarios, such as those with a single output neuron (see Section~\ref{sec:known}).
As our second contribution, we identify a general data throughput condition that allows us to circumvent the aforementioned intractability and guarantee the polynomial-time tractability of \lNNT. 
 In particular, we say that a network admits an \emph{untangling} if a subset of its nodes can be partitioned into connected blocks---one for each output and input node---such that:

\begin{itemize}[noitemsep,leftmargin=10pt,topsep=0pt]
\item[1a.] each input node is the sole source of an input-block;
 \item[1b.] each output node is the sole sink of an output-block; and
\item[2.] the adjacencies between blocks reflects reachability between inputs and outputs.
 \end{itemize}

Intuitively, an untangling can be seen as a backbone in a network where each block either aggregates information from a source or into a sink. While it may not be obvious from the definition at first glance, not only do single-output (as well as single-input) neural networks admit a trivial untangling, but so do many other network architectures one may consider. To exploit our notion of untanglings, we obtain an algorithm that solves \lNNT in polynomial time whenever the network comes with an untangling. We combine this result with a procedure that computes an untangling (or determines that none exists) in linear time on all architectures with a constant number of hidden neurons, and provide a different linear-time procedure to find untanglings on architectures of constant \emph{treewidth}~\citep{RobertsonS83,Bodlaender16a}.

\iflong
\smallskip
\noindent
\textbf{Paper Organization.} After setting up necessary preliminaries in Section~\ref{sec:prelim}, we proceed by formalizing the \NNT\ problem and surveying known results about its complexity in Section~\ref{sec:known}. We then present our new algorithm for training ReLU-activated neural networks with ReLU in Section~\ref{sec:relu}, and follow up in Section~\ref{sec:exten} by a discussion of the result's implications for obtaining efficiently trainable ReLU-activated architectures. Finally, we present our results centered around linear activation functions in Section~\ref{sec:lin}, which details how untanglings can be computed as well as used to efficiently solve \lNNT.
\fi

\section{Preliminaries}
\label{sec:prelim}

\smallskip
\noindent \textbf{Graphs.}\quad
We assume basic knowledge of graph terminology~\citep{Diestel}.
The graphs considered in this paper are assumed to be \emph{directed}, that is, a \emph{graph} is a pair $G = (V,E)$, where $V$ is a set of vertices and $E$ a set of directed edges. 
\iflong
A \emph{path} is a sequence of distinct vertices $v_1,\ldots,v_k$ such that there is a sequence of directed edges $(v_1,v_2),(v_2,v_3),\ldots,(v_{k-1},v_k)$.
We call a graph \emph{acyclic} if there is no sequence of directed edges $(v_1,v_2),(v_2,v_3),\ldots,(v_k,v_1)$ in $G$. 
\fi
For a vertex $v$, the vertices $u$ such that there is an edge $(u,v)$ in $E$ are called the \emph{in-neighbors} of $v$; correspondingly, the set of in-neighbors is the \emph{in-neighborhood} of $v$. The notions of of out-neighbors and out-neighborhood are defined symmetrically. The \emph{in-degree} and \emph{out-degree} of a vertex is the size of its in- and out-neighborhood, respectively. 
For a vertex subset $X\subseteq V$, we denote by $G[X]$ the graph induced on $X$.
 \iflong
For brevity, we will generally refer to directed edges simply as ``edges''; in the few cases we need to refer to edges which are undirected, we make this clear by explicitly calling them ``undirected edges''.
\fi

\smallskip
\noindent \textbf{Neural Networks.}\quad
The protagonists of this article are {neural networks}, which are acyclic graphs endowed with additional information.
In particular, the \emph{architecture} of a neural network $\mathcal{N}$ is a directed acyclic graph $G = (V,E)$ where the nodes are called \emph{neurons}.
 We refer to nodes with no incoming edges as \emph{input nodes}, while those with no outgoing edges are called \emph{output neurons}. All other nodes are called \emph{hidden neurons}. 
 We refer to edges as \emph{deep} when they are not incident to an input node. 
Typically, one assumes that the hidden neurons are partitioned into \emph{layers} defined by their shortest distance to an input node, where a hidden neuron in layer $i$ only has in-neighbors in layer $i-1$ and out-neighbors in layer $i+1$.

In addition to its neural structure, $\mathcal{N}$ is defined by assigning a \emph{bias} $b_v \in \RR$ for each non-input neuron $v$ and a \emph{weight} $a_e \in \RR$ for each edge $e$.
 Furthermore, we fix an \emph{activation function} $\sigma: \RR \rightarrow \RR$ for the entire neural network. 
To formally map the data dimensions onto the input and output nodes in an architecture, we will assume that these nodes come with an implicit predetermined ordering.
 
A neural network defined as above \emph{computes} a function $f_{\mathcal{N}}: \RR^n \rightarrow \RR^d$, where $n$ is the number of input nodes, and $d$ the number of output neurons of $\mathcal{N}$. To define $f_{\mathcal{N}}$, we say for each neuron $v$ that it computes a function $f_v: \RR^n \rightarrow \RR$, which is inductively defined as follows: The function computed at the $i$-th input node is the coordinate projection $\pi_i: \RR^n \rightarrow \RR, (x_1,\ldots,x_n) \mapsto x_i$. 
For a non-input neuron, let $v_1,\ldots,v_t$ be the in-neighborhood of $v$, again in some arbitrary, but fixed order, and let $\mathbf{a}_v$ be the vector of weights $(a_{v_iv})_{i=1,\ldots,t}$ in the same order. Furthermore, let $\mathbf{y}_v = (y_1,\ldots,y_t)$ be the values of the functions computed at the neuron $v_i$, respectively---that is, $y_i = f_{v_i}(x_1,\ldots,x_n)$. Then, $v$ \emph{computes} $\sigma(\mathbf{a}_v^T \cdot \mathbf{y}_v + b_v)$.
The function computed by the neural network $\mathcal{N}$ itself is then given by $(f_{o_1}(x_1,\ldots,x_n),\ldots,f_{o_d}(x_1,\ldots,x_n))$, where $o_1,\ldots,o_d$ are the output neurons of $\mathcal{N}.$

Clearly, the function computed by $\mathcal{N}$ depends, among other things, on the employed activation function. Here we will consider two specific and well-studied choices for $\sigma$, namely \emph{ReLU} and \emph{linear} activations. 
It is well known that, without loss of generality, these can be assumed to be $\sigma(x) = x$ (i.e., the identity) for linear activations and $\sigma(x) = \max\{0,x\}$ for ReLU activations~\citep{DeepL}.
 We will also employ the notation $x^+ = \max\{0,x\}$.

\smallskip
\noindent\textbf{Training Neural Networks.}\quad
The task of \emph{training} a neural network is to, given an architecture and activation, compute a set of weights as well as biases, such that the function computed by the resulting neural network approximates in the best possible manner a given set of training data, that is, a given set of samples $(x,y)$ where $x$ and $y$ are vectors of the appropriate dimension for the network architecture.

More formally, the input to the $\sigma$-$\NNT$ problem 
 is an architecture $G = (V,E)$ and data points $D \subseteq \RR^n \times \RR^d$, where $n$ is the number of input and $d$ the number of output neurons of the architecture. For every setting $a$ of all edge weights as well as $b$ of all biases, this fully determines a neural network $\mathcal{N}_{a,b}$ that in turn computes a function $f_{a,b} = f_{\mathcal{N}_{a,b}}$.
The task is then to output a weight $a_e$ for every edge $e$ in $G$ as well as a bias $b_v$ for every neuron such that the function computed by the resulting neural network $\mathcal{N}_{a,b}$ minimizes, among all such choices of $a$ and $b$, the $\ell^2_2$-loss function
\iflong
\begin{align} \label{eq:loss}
L_D(a,b) = \sum_{(x^\ast, y^\ast) \in D} || y^\ast - f_{a,b}(x^\ast)||_2^2.
\end{align}
Here, 
\fi
\ifshort
$L_D(a,b) = \sum_{(x^\ast, y^\ast) \in D} || y^\ast - f_{a,b}(x^\ast)||_2^2,$ where
\fi
$||\cdot ||_2$ is the standard Euclidean norm.
  \iflong We remark that for complexity-theoretic reasons, $\sigma$-\NNT\ can also be phrased as an equivalent decision problem where the input additionally contains a \emph{threshold} $\theta$ and the question is whether we can find $a,b$ such that $L_D(a,b) \leq \theta$; in particular, this is used in Section~\ref{sec:known}.
\fi
        
Note that the problem formalization in particular entails that we do not regard the type of activation function as part of the input, whereas the network architecture \emph{is} part of the input by default. This either matches or generalizes previous formalizations of the problem~\citep{AroraBMM18,AbrahamsenKM21,GoelKM021,FroeseHN22,Fullyconnet}. We will denote the problems of training neural networks for linear and ReLU activations as \lNNT and \rNNT, respectively, and the analogous problem where the network contains both linear and ReLU activation functions as \mNNT. Furthermore, while the definition of $\sigma$-$\NNT$ implies that a \emph{solution} is one which minimizes the $\ell^2_2$-loss function, we will also sometimes call these ``optimal solutions'' for emphasis. Last but not least, throughout the article we assume that $|D|\geq 2$ since the case where $|D|\leq 1$ is trivial.

\section{Synopsis of Known Results}
\label{sec:known}
Neural network training is a vast area of research that has received an immense amount of attention both on the theoretical as well as the empirical side. In order to put our results into context, we survey the complexity-theoretic understanding 
 of linear and ReLU neural network training. 
 
\ifshort
\smallskip
\noindent \textbf{Linear Activation.}\quad
 In the simplest case of \lNNT where we assume the size of the whole architecture to be upper-bounded by a constant, it is known that the problem can be expressed as a constant-size system of polynomial inequalities which can be solved in polynomial time, e.g., via the seminal work of \citet{tarski}.  
On the other hand, \lNNT is known to be intractable even if the number of output neurons is upper-bounded by a fixed constant; in fact, the problem was shown to be \ER-hard even when the number of output neurons is fixed to $3$~\citep{AbrahamsenKM21}.

One intermediate case between the general and fixed-size fixed-dimension networks is when the network architecture is unrestricted in size but has a special structure, namely, where all edges between layers of nodes exist. 
In this case, \lNNT\ can be solved in polynomial time by a direct application of multidimensional linear regression~\citep{velu2013multivariate}.  \fi

\iflong
\subsection{Linear Activation}
\label{sub:lin}
Despite our considering the architecture part of the input, 
it does also make sense to ask what happens if we assume it fixed, much like the activation $\sigma$. 
While it is one of the practitioner's central tasks to first design an architecture for a neural network 
(and in this sense, it may not be as fixed as, say, activation functions, which are rarely re-designed), 
and then train it, such a design is much less ephemeral than the other part of the input, namely the data set $D$. 
The simplest question, therefore, that should be answered, is: Is there an efficient algorithm for $\lNNT$ if the network architecture is of fixed size as well?

\subsubsection{Fixed Architecture Size and Fixed Input Dimension}
As it turns out, already this case requires some thought: It is easy to argue inductively that, if $\sigma$ is the identity function, the function computed by the neural network will be a polynomial in $x$, the inputs (in fact, $x$ has degree at most $1$), \emph{as well as} the set of weights $a$ and biases $b$. That is, if we consider every weight and bias an indeterminate, then $f_{\cdot,\cdot}$, where we left $a$ and $b$ unspecified, will be a formal polynomial in these indeterminates, in addition to the inputs $x$. 
In turn, also the loss function $L(\cdot,\cdot,D)$ becomes a formal polynomial in only the indeterminates $a,b$ (because the inputs $x$ are evaluated at each data point $D$, $x$ disappears as an indeterminate, and becomes determinate) of degree only depending on the network architecture. Note, however, that the encoding length of the coefficients appearing in this polynomial, say, $\lambda$, will depend polynomially on the encoding size of the data $D$.
By virtue of the polynomiality of $L(\cdot,\cdot,D)$, asking for a setting of $a,b$ such that $L(a,b,D) \leq \theta$ holds is formally captured by the sentence
\[
\exists a,b: L(a,b,D) \leq \theta
\]
Ever since the seminal work of \citet{tarski}, it has been known that the truth of such sentences (that is, deciding their membership in the so-called \emph{existential theory of the reals}) can be determined within a number of steps that is polynomial in the encoding length $\lambda$ of the coefficients (as well as the encoding length of $\theta$), but exponential in the degree and number of variables of the sentence. Since the latter two quantities are fixed in our case by assumption, the formulation above yields a polynomial-time algorithm for linear $\NNT$ when the network architecture is entirely fixed.

\subsubsection{The General Case}
At the other extreme, if we allow both the number of hidden and input nodes to grow arbitrarily and the architecture of the network to be arbitrary as well, the problem becomes intractable even for linear activation functions~\citep{AbrahamsenKM21}. In particular, \lNNT is \ER-hard, which roughly means that training a neural network of the aforementioned type to optimality is as difficult as solving an arbitrary given system of polynomial inequalities (which is strongly believed to not be polynomial-time tractable). 
  
\subsubsection{Fully Connected Layered Networks}
One intermediate case between the general and fixed-size fixed-dimension networks is when the network architecture is unrestricted in size, but of a special structure, namely, we assume all edges between layers of vertices to exist. 
Now, training fully connected networks with linear activations under mean squared errors is nothing else than multidimensional linear regression. First, observe that if two sets of $n$ inputs and $m$ output neurons with linear activations are fully connected by a single layer of edges from inputs to outputs, then such a neural network computes some affine linear mapping $x \mapsto Ax + b$, with $A \in \RR^{m\times n}$, where the $(i,j)$-entry $A_{i,j}$ of $A$ is given by the weight on the edge between the $i$-th input and the $j$-th output neuron, and the $j$-th bias $b_j$ is the sum all biases of the inputs and the $j$-th output bias. Computing such weights is the aim of classic multidimensional linear regression, and there are efficient and explicit formulations for the optimal estimators in this setting~\citep{velu2013multivariate}.

By the same token, if we add a second layer of edges to this example, the computed mapping  becomes $x \mapsto A_1 A_2 x + b$, where $A_2$ contains the weights of the edges adjacent to the inputs, and $A_1$ those adjacent to the outputs (and the biases arise similarly as before).
If we denote with $r$ the cardinality of the hidden layer of neurons, hence $A_2 \in \RR^{n\times r}, A_1 \in \RR^{r\times m}$, it is evident that $A_2 A_1$ is a matrix of rank at most $r$. In turn, all rank-$\leq r$ matrices can be represented in this way for some choice of $A_1,A_2$. 
By asking for a setting of weights and biases that minimize the sum of squared errors, we are thus asking for the best rank-$\leq r$ matrix under the squared mean error --- so-called multivariate reduced-rank regression, which is an important and well-studied special case of multivariate regression that admits explicit solutions \citep{velu2013multivariate}.
Now, importantly, for more than a single layer of hidden neurons, the problem remains the same. 
Indeed, whenever we have several such layers of varying sizes, the smallest among them, call its size again $r$, determines the maximal rank of the sought solution, so that the problem in its full generality remains efficiently solvable, as long as all edge layers are fully connected.

\subsubsection{Relation to Our Results}
\fi

The preceding discussion leaves open a variety of possible directions, such as the complexity of training networks that are not necessarily fully connected, but constrained to have two or even a single output neuron.
On the other hand, one may ask what happens if we consider only the number of hidden neurons to be bounded by a constant, without restricting the number of input nodes and output neurons.
 Our results deal with an intermediate case that captures, among other things, the easy case of single-output networks, as well as some more general networks with a constant number of hidden neurons.
In particular, we will formulate a technical condition---whether or not the architecture can be \emph{untangled}---that will allow us to train such networks with linear activations in polynomial time.

\ifshort
\smallskip
\noindent \textbf{Rectified Linear Activation.}\quad
\fi
\iflong
\subsection{Rectified Linear Activation} 
\fi
 Recent results from computational complexity theory provide strong lower bounds on the inherent difficulty of the problem of neural network training with ReLU activation functions, already on highly restrictive architectures. Indeed,\citet{GoelKM021} and independently \citet{DeyWX20} have proven that, under standard complexity assumptions,  even training a single ReLU-activated neuron with $n$ input nodes cannot
 be accomplished in time polynomial in $n$ and the size of the training set $D$; see also the related lower bounds of \citet{Fullyconnet}.
  
A seminal algorithmic result is that of \citet{AroraBMM18} for training shallow networks consisting of one layer of hidden ReLU neurons and a single output neuron equipped with a linear activation function.
Their algorithm runs in time $|D|^{\bigoh(w\cdot n)}$, where $D$ is the training set, $n$ the number of input nodes and $w$ the number of other nodes in the architecture. \citet{BoobDL22} showed that the same approach can be used to handle cases where the output neuron is equipped with a ReLU activation function as well.
\citet{FroeseHN22} have shown that this running time dependency on $n$ in the exponent of $D$ is asymptotically optimal even for a single hidden neuron, ruling out algorithms running in time, e.g., $2^n\cdot D^{o(n)}$ for this basic case.

It is a major open problem to extend the techniques of \citet{AroraBMM18} to general deep networks---in particular, the complexity of \rNNT\ when restricted to networks whose size is upper-bounded by an arbitrary but fixed constant is wide open. 
Our work shows that it is nevertheless possible to identify a much more general set of architectures than those covered by the aforementioned results~\citep{AroraBMM18,BoobDL22} for which \rNNT\ is polynomial-time solvable.
     
\section{Training ReLU Networks}
\label{sec:relu}
In this section, we establish our tractability results for ReLU-activated neural network training when each hidden neuron has out-degree at most $1$.
As our first step, we prove that in such instances the weights in the hidden layers can be discretized to only two values.
  
\iflong
\begin{lemma}
\fi
\ifshort
\begin{lemma}
\fi
\label{lemma:slide_local_relu}
For any ReLU network $\mathcal{N}$ and for any hidden neuron $u$ with precisely one out-neighbor $w$, there is a ReLU network $\mathcal{N}^*$ with the same architecture such that $f_\mathcal{N}=f_{\mathcal{N}^*}$ and $a^*_{uw}\in \{-1,1\}$.
Moreover, the only parameters in which $\mathcal{N}$ and $\mathcal{N}^*$ differ are the weights of the edges incident to $u$ and the bias of $u$.
\end{lemma}

\ifshort
\begin{proof}[Proof Sketch]
 Let $(v_i)_i$ be the (non-empty) family of predecessors of $u$, and $w$ its only successor.
To obtain $\mathcal{N}^*$, we construct a new set of weights and biases at $u$, where $\mathrm{sgn}(x)$ is the sign function:
\begin{align*}
\forall i, a^*_{v_i u} &=a_{v_i u} \cdot |a_{uw}| \\
b^*_u &=b_u \cdot |a_{uw}|\\
a^*_{uw} &=\mathrm{sgn}(a_{uw})  \quad ( 1 \text{ if }a_{uw}=0 ) 
\end{align*}
All other weights in $\mathcal{N}^*$ remain the same as in $\mathcal{N}$. To complete the proof, it suffices to show that $f_\mathcal{N}=f_{\mathcal{N}^*}$ holds by using the fact that the ReLU function commutes with multiplication with non-negative factors.
\end{proof}
\fi

\iflong
\begin{proof}
 Let us take an arbitrary ReLU network and $u$ a hidden neuron of out-degree $1$.
Let us call $(v_i)_i$ the (non-empty) family of predecessors of $u$, and $w$ its only successor. Let us set:
\begin{align*}
\forall i, a^*_{v_i u} &=a_{v_i u} \cdot |a_{uw}| \\
b^*_u &=b_u \cdot |a_{uw}|\\
a^*_{uw} &=\mathrm{sgn}(a_{uw})  \quad ( 1 \text{ if }a_{uw}=0 ) 
\end{align*}
We claim that these modifications do not change the function computed by the network. Indeed, let us prove that for all possible inputs the value received by $w$ remains the same.

\begin{align*}
y^* &=a^*_{uw} \cdot \left( b^*_u + \sum_i a^*_{v_iu}\cdot y_i \right)^+ \\
&=a^*_{uw} \cdot \left( |a_{uw}|\cdot (b_u + \sum_i a_{v_i u} \cdot y_i) \right)^+,
\end{align*}
and hence
$$
y^*=a_{uw}\cdot \left( b_u + \sum_i a_{v_i u} \cdot y_i \right)^+ =y
$$
Indeed, the ReLU operator commutes with multiplication by a \textit{non-negative} factor.
The rest of the network remains exactly the same, thus the network computes the same function.
\end{proof}
\fi

By iteratively applying Lemma~\ref{lemma:slide_local_relu}, we obtain:

\iflong
\begin{lemma}
\fi
\ifshort
\begin{lemma}
\fi
\label{lemma:slide_global_relu}
Let $\mathcal{N}$ be a ReLU structure such that every hidden neuron has out-degree $1$. 
Then there exists an optimal solution to \rNNT such that all deep edges have weights in $\{1,-1\}$.
\end{lemma}

\iflong
\begin{proof}
From any optimal solution, using Lemma~\ref{lemma:slide_local_relu} on every hidden neuron, starting with the deepest layers, immediately yields a different solution with the desired properties, and which is also optimal because it computes the same function.
\end{proof}
\fi

Lemma~\ref{lemma:slide_global_relu} will later allow us to reduce our search space for \rNNT by only needing to consider two options for the weights of each deep edge. 
 
\paragraph{Partitioning by Hyperplanes.}
The other key argument used to reduce the search space further concerns the way ReLU neurons partition the training set into a ``dead area'' (i.e., the elements where the linear unit is rectified) and an ``active area'' (where the unit behaves as a linear function). 
More precisely, given a data set $D$ and a ReLU-activated neuron $u$, we say that $u$ \emph{partitions} $D$ into
\ifshort
an \emph{active area} $D|_u$, which contains all data points for which $f_u$ outputs a value greater than zero, and a \emph{dead area} $\overline{D|_u}=D\setminus D|_u$, which contains all data points for which $f_u$ outputs zero.
\fi
\iflong
\begin{itemize}
\item an \emph{active area} $D|_u$, which contains all data points for which $f_u$ outputs a value greater than zero, and
\item a \emph{dead area} $\overline{D|_u}=D\setminus D|_u$, which contains all data points for which $f_u$ outputs zero.
\end{itemize} 
\fi

As a base case, let us consider a 
neuron $u$ in the first hidden layer of a neural network and let $\mathbf{a}_u^T$ be the vector of weights for the edges from the input nodes to $u$, where a non-edge would be represented as a weight of $0$. As a trivial upper bound, there are in general at most $2^{|D|}$ ways $u$ can partition $D$ into an active and dead area, which provides a trivial upper bound for an algorithm enumerating all such partitions. However, we have that $f_u(D_i)=(\mathbf{a}_u^T \cdot D_i + b_u)^+$ for each data point $D_i\in D$ and moreover $\mathbf{a}_u^T \cdot \mathbf{x} + b_u$ defines a hyperplane in an $n$-dimensional space.
This allows us to obtain an algorithm with a better running time bound for enumerating the partitions of the point set that is inspired by a result of~\cite{Megiddo88}.

\iflong
\begin{lemma}
\fi
\ifshort
\begin{lemma}
\fi
\label{lem:enum_partitions}
Let $D \subseteq \RR^n$ be a finite set of points.
Then, the set of partitions $\Pi = \{(A, B) \mid A \cup B = D $ and $A, B$ are separated by a hyperplane$\}$ can be enumerated in time $|D|^{\bigoh(n)}$.
\end{lemma}
\iflong
\begin{proof}
The algorithm proceeds by enumerating a superset of all such possible partitions, which in total will still obey the desired bound, whereas we then discard superfluous elements by checking each candidate partition.
To this end, note that for any $\pi = (A,B) \in \Pi$, it is the case that there must exist a normal vector $\mathbf{v}_\pi \in \RR^n$, an offset $b_\pi \in \RR$ as well as some maximal slack value $r_\pi \geq 0$ such that $\mathbf{v}_\pi^T \mathbf{a} \geq b_\pi + r_\pi$ for all $\mathbf{a} \in A$ and $\mathbf{v}_\pi^T \mathbf{b} \leq b_\pi - r_\pi$ for all $\mathbf{b} \in B$. By rescaling, we may assume w.l.o.g.\ that $\mathbf{v}_\pi \in [-1,1]^n$.
In particular, $\mathbf{v}_\pi, b_\pi$ and $r_\pi$ can be chosen as an extremal point of the convex polyhedron defined via the following linear program over the variables $v_1,\ldots,v_n,b,r \in \RR$:
\begin{align*}
	\max r & \text{ subject to } \\
	\mathbf{v}^T \mathbf{a} &\geq b + r &\text{ for all } \mathbf{a} \in A,\\	
	\mathbf{v}^T \mathbf{b} &\leq b - r &\text{ for all } \mathbf{b} \in B,\\	
	-1 \leq v_i &\leq 1 &\text{ for all } 1 \leq i \leq n,\\
	r &\geq 0.
\end{align*}
Note that if $A$ and $B$ are not separated by a hyperplane, then this system of linear inequalities will simply be infeasible, and moreover, the system is always bounded.
As is well known from polyhedral theory, extremal points of bounded polyhedra in $\RR^{n+2}$ are defined by a set of $n+2$ linearly independent inequalities in the linear program defining the polyhedron, where these inequalities hold with equality at the given extremal point.
   Any set of $n+2$ linearly independent tight inequalities is determined by a selection of at most $n$ points $A' \cup B'$ with $A',B' \subseteq D$, at most $n$ choices of $v_i = \pm 1$ and whether or not $r = 0$. Note that this is true \emph{independent} of the choice of $A$ and $B$. Conversely, any such selection determines a set of tight inequalities, which may possibly be linearly dependent and possibly inconsistent. However, observe that the number of such choices is still bounded by $|D|^{\bigoh(n)}$, and each choice determines a linear system of $n+2$ equations and variables that can be checked for consistency and uniqueness in time $n^{O(1)}$. 

The algorithm hence proceeds by (1) inspecting one after another all the $|D|^{\bigoh(n)}$ choices of tight inequalities (which implicitly includes all those choices induced by the actual partitions $(A,B) \in P$), (2) for each choice, checking the resulting system for a unique solution $\mathbf{v},b,r$ and finally (3) computing the signs of $\mathbf{v}^T\cdot \mathbf{d} - b$ for each $\mathbf{d} \in D$ to determine the resulting partition of $D$, and outputting it. The correctness of this procedure follows by construction, as does the claimed running time.
\end{proof}
\fi

\smallskip
\noindent \emph{Remark.}\quad A similar result was claimed and used by \citet[Page 14]{AroraBMM18} in their algorithm for dealing with \rNNT restricted to a special case of the setting treated in this section. However, the argument presented there seems to be incomplete. In particular, it makes the claim that ``the total number of possible hyperplane
partitions of a set of size $D$ in $\RR^n$ is at most $2 {D\choose n}$'', which is imprecise (consider, e.g., $|D|=n$). Moreover, the book listed as reference for bounding the total number of possible hyperplane partitions does not provide an algorithm for enumerating these efficiently, and enumeration is required for both our result and the result claimed in \cite{AroraBMM18}. A bound on the number of such partitions can be attributed to \cite{harding1967number}, but that does not yield efficient enumeration either. If the points are known to lie in general position, efficient enumeration could be carried out via a translation to hyperplane arrangements followed by an application of the results of \cite{EdelsbrunnerOS86}; however, in our setting we cannot guarantee that these lie in general position, necessitating our stand-alone proof of Lemma~\ref{lem:enum_partitions} above (which we believe to be also of general interest).

 The considerations preceding Lemma~\ref{lem:enum_partitions} do not immediately translate to neurons that are deeper than in the first layer, since the inputs to these activation functions are obtained by non-linear transformations of $D$. Nevertheless, we prove that it is also possible to derive an upper bound for neurons beyond the first hidden layer if we are provided information about the partitions of neurons in previous layers.
  
\iflong
\begin{lemma}
\fi
\ifshort
\begin{lemma}
\fi
\label{lemma:cells_partition}
Let $D$ be a set of data points, $u$ be a ReLU neuron, $F$ be the set of all neurons on paths from input nodes to $u$, and $x=|F|$. Given $D|_v$ for each $v\in F$, 
we can upper-bound the number of distinct active areas $D|_u\subseteq D$ over all networks consistent with the given selection of $\{D|_v~|~v\in F\}$ by
$ |D|^{\bigoh\left(n \cdot 2^x\right)}$, and it is possible to enumerate a superset of these in time $|D|^{\bigoh\left(n\cdot 2^x\right)}$.
\end{lemma}

\iflong
\begin{proof}
Let us start by defining the equivalence relation $\mathcal{R}$ on elements of $D$ induced by the $\{D|_v~|~v\in F\}$. Namely, two data points are equivalent for $\mathcal{R}$ if and only if, for every neuron $v$ in $F$ they either both belong to $D|_v$ or they both don't. This is trivially a equivalence relation, and since $|F|=x$, the number of equivalence classes is bounded by $2^x$. Let us call these equivalence classes $(A_i)_{0\leq i<2^x}$, and let $(a_i)_{0\leq i<2^x}$ be their respective sizes. \\
Now, we focus on each such equivalence class $A_i$ separately. For any pair of data points $d_1, d_2$ in the same $A_i$, we know by definition that the set of activated neurons on the way to the neuron $u$ is the same. Since the other operations in the network (biases and weights) are affine transformations, $d_1$ and $d_2$ undergo the same affine transformation, which can be written $h_i(d)=a^Td+b$ for some vectors $a$ and $b$, on their way to $u$.
Since this is the signal that reaches $u$, we know that the limit between the active and dead areas of $u$ is defined by $h_i(d)=0$, which defines a hyperplane in the original n-dimensional input space. Crucially, while we do not know what the values of $a$ and $b$ are, the number of possible ways this hyperplane can bipartition the data points in $A_i$ can be upper-bounded by $|A_i|^{\bigoh(n)}$ (which is, in turn, at most $|D|^{\bigoh(n)}$), and the set $P_i$ of these bipartitions can also be enumerated in the same asymptotic running time by using Lemma~\ref{lem:enum_partitions}.

In order to apply the above to obtain a set of bipartitions of all data points (not just those in a single $A_i$) into those which activate $u$ and those which do not, we define the set of bipartitions of $D$ by considering every combination of a possible bipartition from each of $P_1, P_2,..., P_j$---more formally, this is the set $\mathcal{P}=\{ (p, D\setminus p) | p \subseteq D, \exists p_1\in P_1…p_j\in P_j: p= p_1 \cup p_2 \cup … \cup p_j \}$. \\
Remark that the size of $\mathcal{P}$ is trivially upper bounded by $|D|^{\bigoh\left(n\cdot 2^x\right)}$ since all $P_i$ are of size upper bounded by $|D|^{\bigoh(n)}$. This completes the construction required for Lemma~\ref{lemma:cells_partition}.
     \end{proof}
\fi

We now show how to solve $\rNNT$ in polynomial time for 
each fixed value of $n$ and $w$:
 \iflong
\begin{theorem}
\fi
\ifshort
\begin{theorem}
\fi
\label{th:relu}
There is an algorithm that, given an instance $(G, D)$ of \rNNT such that every hidden neuron of the structure has out-degree exactly $1$, and the data points are encoded using $L$ bits overall, computes the global optimal solution in time $|D|^{\bigoh\left(n\cdot w\cdot 2^w\right)} \cdot \poly(L)$.  
 \end{theorem}

\ifshort
\begin{proof}[Proof Sketch]
The algorithm begins by exhaustively branching to determine the weight of each deep edge, which requires us to consider at most $2^w$ options in view of Lemma~\ref{lemma:slide_global_relu}. Once the weights of the deep edges are fixed, we branch to determine how each hidden and output neuron partitions the training set. For this to work, we need to process layers consecutively---we only branch on the possible partitions for one neuron after we have chosen the partitions for all of its hidden in-neighbors.

For each such neuron, we enumerate all feasible partitions of $D$ using Lemma~\ref{lemma:cells_partition}, and branch over the possibilities. For a given partition, we create a set of inequalities that are satisfied only if each element of $D$ indeed is affected by $u$ according to the partition. Once a partition has been chosen for all neurons, the function computed by the resulting neural network $\mathcal{N}$ for each neuron is fixed save for the weights of edges adjacent to input nodes, and the biases. We will optimize over all these free variables by constructing an instance of the \textsc{Quadratic Optimization} problem, whereas we optimize $L_D(a,b)$ subject to the set of constraints obtained when partitioning the neurons. The obtained instances of \textsc{Quadratic Optimization} can be solved in time polynomial in the number of variables and constraints, as well as the bitlengths of the datapoints, e.g., by using the ellipsoid method~\citep{Pang83,Kozlov79}.
\end{proof}
\fi

\iflong
\begin{proof}
The first step of the algorithm is to guess (i.e., exhaustively branch over) a weight within $\{-1,1\}$ for each deep edge. We know with Lemma~\ref{lemma:slide_global_relu} that it is safe to do so. For every one of the possibilities (bounded by $2^w$ because hidden neurons only have one outgoing edge and output neurons have none), we will compute the optimal solution consistent with this choice of weights for deep edges, and then we take the best over the possibilities.

Now that the weights of the deep edges are fixed, we will guess for each hidden and output neuron the way it partitions the training set. For this to work, we need to take into account the structure of the network: we will branch on the possible partitions for one neuron only once we have chosen the partitions for all of its hidden in-neighbors. This is always possible because we don't have any cycles in the network.

For each such neuron, we enumerate all feasible partitions of $D$ using Lemma~\ref{lemma:cells_partition}, and branch over the possibilities. For a given partition, we create a set of inequalities that are satisfied only if each element of $D$ indeed is affected by $u$ according to the partition: For each $\delta\in D$, we create an inequality of the form $g_\delta(\delta) \geq 0$ if $\delta$ is in the active part of the partition, and $g_\delta(\delta) \leq 0$ otherwise. For any $\delta$, $g_\delta$ is an affine function whose expression depends on the way $\delta$ was affected by each neuron encountered before $u$. Namely, for two data points $\delta$ and $\delta'$ which up to this point were always in the same partition, we have: $g_\delta=g_{\delta'}$. This function also depends on the weights of non-deep edges and biases, as will be important in the use of the inequalities below.

Once a partition has been chosen for all neurons, the function computed by the resulting neural network $\mathcal{N}$ for each neuron is fixed save for the weights of edges adjacent to input nodes, and the biases. We will optimize over all these free variables using a quadratic optimization problem. This problem will be: optimizing $L_D(a,b)$ subject to the set of constraints we built when partitioning the neurons. We remark that if we considered a configuration of partitionings that can't be realized, the set defined by the constraints will be empty. In this case, we simply discard this configuration.

Since the cost function is least-squares, this means that the problem is solvable in polynomial time (in the number of variables and constraints), for instance with the ellipsoid method. 
The number of variables is $\bigoh(n\cdot w)$ and the number of constraints $|D|\cdot w$, thus the quadratic optimization instance can be solved in time polynomial in $|D|$, $n$ and $w$~\citep{Pang83,Kozlov79}.

Overall, since we branch on $\bigoh(2^w)$ possibilities for the signs of deep edges, then $|D|^{\bigoh\left(n\cdot 2^w\right)}$ for each of the $w$ ReLU neurons, meaning we solve in total $|D|^{\bigoh\left(n\cdot w\cdot 2^w\right)}$ quadratic programs, for a total running time of $|D|^{\bigoh\left(n\cdot w\cdot 2^w\right)}$.
 \end{proof}
\fi

Note that the proof of Theorem~\ref{th:relu} can be directly extended to also include architectures where some of the hidden or output neurons use a linear activation function instead of ReLU. 
 \ifshort
Indeed, we need not guess a partition for the neuron and hence not construct a constraint in the quadratic program, and furthermore we may treat a linear neuron the same way as an activated ReLU in the following layer. 
 \fi
\iflong

To be more precise, consider what steps are undertaken in the proof when considering a ReLU-activated neuron: First, we need to guess on which part of the data this neuron is activated. For linearly activated neurons, this question is moot: There is no cut-off to be considered in this case. Therefore, there is also no constraint that we could add to the quadratic program that we construct in the proof (after all, the ``active region'' of a linearly activated neuron may be considered to be its entire input space). Therefore, also, the result of the computation of a linearly activated neuron will behave much like the result of the computation output at a ReLU-activated neuron in its active area, in the sense that it is always given as a linear combination of its respective inputs. To summarize, by simply not performing some of the processing tasks necessary for ReLU-activated neurons, we can also capture the networks with linearly activated neurons.

\fi
This yields the following theorem, which generalizes the tractability results of both \citet{AroraBMM18} and \citet{BoobDL22}.

\begin{corollary}
There is an algorithm that, given an instance $(G, D)$ of \mNNT such that every hidden neuron of the structure has out-degree $1$, and the data points are encoded using $L$ bits overall, computes the global optimal solution in time $|D|^{\bigoh\left(n\cdot w\cdot 2^w\right)} \cdot \poly(L)$.
\end{corollary}

\section{Blowing Up a Neural Network}
\label{sec:exten}
The tractability results of \citet{AroraBMM18} on constant-width shallow ReLU-networks raise the natural question of whether it is possible to (efficiently) perform depth reduction on deeper networks, 
so as to reduce a given (deep) instance for $\rNNT$ to a tractable (e.g., shallow) special case. It is widely conjectured that this is not possible in general without sacrificing some quality of the solution (see, e.g., \cite{HertrichDiss}, Conjecture 3.1. and~\cite{iclr/0001HL23,HertrichBSS21}). Among others, it is widely believed that there are functions that are efficiently computable by depth-3 but not by depth-2 ReLU-activated neural networks~\citep{EldanS16}, and similar results were also shown for deeper networks~\citep{CohenSS16,Telgarsky16}. 

While these results essentially rule out performing depth reduction without a negative impact on the accuracy, our tractability results in Section~\ref{sec:relu} allow us to approach the question from a different angle. In particular, we show that it is possible to transform an arbitrary constant-size ReLU-activated deep neural network architecture into one of the same depth which is (1) at least as powerful as the original architecture, (2) guaranteed to admit polynomial-time training, and (3) upper-bounded in size by a function of the size of the original architecture\footnote{We remark that such a transformation can be trivially provided for linearly-activated network architectures, since \lNNT is polynomial-time solvable on complete networks. We also explicitly note that changing the architecture is not always desirable or even possible.}.

To explain the construction, we start again with the base case of a single neuron. More precisely, let $v$ be a single ReLU-activated (hidden) neuron with $n$ inputs and $d$ outputs, and say the weights on the input edges are $a_1,\ldots,a_n$, the weights on the output edges are $a'_1,\ldots,a'_d$, and the bias at the hidden neuron is $b$. 
Now consider a new network with a single layer of ReLU-neurons which consists of $d$ ReLU-activated neurons $v_1,\ldots,v_d$, where every $v_i$ is connected to all inputs with the same weights $a_1,\ldots,a_n$ as before. Furthermore, each $v_i$ has a single output edge with weight $a'_i$. We call this network the \emph{blow-up} of $N$.
Let us first establish that this blow-up has the desired property:

\iflong
\begin{lemma} 
\fi
\ifshort
\begin{lemma}
\fi
\label{lemma:blow-up-single}
Let $N$ be a neural network with a single ReLU-activated hidden neuron as before, and let $N'$ be its blow-up. 
Let $f_N,f_{N'}: \RR^n \rightarrow \RR^d$ be the functions computed by $N$ and $N'$, respectively.
Then, $f_N = f_{N'}$ holds.
\end{lemma}
\iflong
\begin{proof}
Consider the $i$-th neuron $v_i$. Its inputs are, by construction, the same as the inputs of the hidden neuron, call it $v$, in the original network. Furthermore, since the bias and the weight on its unique output edge is the same as the weight on the $i$-th output edge of the original hidden neuron $v$, the function $v_i$ computes on its single output necessarily is the same as the function $v$ computed on its $i$-th output.
\end{proof}
\fi
This procedure can be generalized to blow up an arbitrary ReLU-activated neural network, regardless of the number of its layers.
\ifshort

\fi
\iflong
Beginning with the last hidden layer, we perform a blow-up on each hidden neuron. That is, more formally, the blow-up of the full neural network architecture is obtained as the result of the following procedure: While there are hidden neurons of out-degree greater than one, replace one of the right-most (that is, in the hidden layer closest to the output layer) hidden neurons with out-degree greater than one with its blow-up. Once there are no hidden neurons with out-degree greater than one, output the current network.

By iteratively applying Lemma~\ref{lemma:blow-up-single}, we obtain:

\begin{lemmabis}
Let $\mathcal{N}$ be a ReLU-activated neural network and $\mathcal{N}'$ its blow-up. 
Let $f_{\mathcal{N}},f_{\mathcal{N}'}: \RR^n \rightarrow \RR^d$ be the functions computed by $\mathcal{N}$ and $\mathcal{N'}$, respectively.
Then, $f_{\mathcal{N}} = f_{\mathcal{N}'}$ holds.
\end{lemmabis}
\fi
With a slight abuse of notation, we will also refer to the \emph{blow-up} of a given network architecture $G$ (that is, one without weights and biases) as the architecture of the blow-up of an arbitrary neural network with architecture $G$, and the process for obtaining the blow-up of an architecture is the very same as for a weighted network (ignoring any weights). 
We now turn to establishing the properties of blow-ups with respect to $\rNNT$.
\iflong
\begin{proposition} 
\fi
\ifshort
\begin{proposition}
\fi
\label{prop:blow-up-opt}
Let $G$ be an arbitrary architecture and let $G'$ be its blow-up.
Furthermore, let $L_D$ and $L'_D$ be the loss functions associated with $G$ and $G'$, respectively.
Then, for any weights $a,b$ on $G$ and the minimizer $a^\ast,b^\ast$ of $L'_D$ on $G'$, we have $L'_D(a^\ast,b^\ast) \leq L_D(a,b)$.
\end{proposition}
\iflong
\begin{proof}
By construction, simply setting the weights on $G'$ to be the weights of the blow-up of the network obtained from putting weights $a,b$ on $G$ gives a network with architecture $G'$ and weights $a',b'$ that satisfies $L'_D(a',b') = L_D(a,b)$.
Hence, in particular for the optimal choice of $a^\ast,b^\ast$ for weights on $G'$ it must hold that $L'_D(a^\ast,b^\ast) \leq L'_D(a',b') = L_D(a,b)$.
\end{proof}
\fi
The size of the blow-up of a network with a constant number of hidden neurons behaves as follows.

\iflong
\begin{proposition}
\fi
\ifshort
\begin{proposition}
\fi
Let $\mathcal{N}$ be a neural network with $\ell$ layers of $\lambda \geq 2$ hidden neurons each, and $o$ output neurons. Then the blow-up of $\mathcal{N}$ has at most $b \leq o \cdot \lambda^{\ell+1}$ hidden neurons, and at most $o \cdot \lambda^\ell$ neurons in its first hidden layer. If $\lambda = 1$, we have $b \leq o \cdot \ell$.
\end{proposition}
\iflong
\begin{proof}
For $\lambda = 1$, the statement becomes obvious.
For $\lambda \geq 2$, the proof is by induction on the depth $\ell$.
If $\ell = 1$, then there are at most $o$ new neurons generated for each of the at most $\lambda$ neurons in the single layer, giving $\lambda\cdot o$ new neurons at most, as desired.
For $\ell > 1$, consider the neural network restricted to the last $\ell-1$ layers, and perform the blow-up on this neural network of depth $\ell-1$. Now, by inductive hypothesis, there are now at most $o\lambda^{\ell-1}$ neurons in the second layer (that is, first layer of the restriction), and by construction, the first hidden layer will have its outputs duplicated so that there are now at most $o\cdot \lambda^{\ell-1}$ outputs for each of the at most $\lambda$ hidden neurons in the first layer. By definition of the blow-up,
this yields at most $\lambda \cdot o\cdot \lambda^{\ell-1} = o \cdot \lambda^\ell$ neurons in the first layer of the blow-up. Summing up gives the bound on the total size (assuming $\lambda \geq 2$).
\end{proof}
\fi

Recalling Theorem \ref{th:relu} from Section~\ref{sec:relu}, we now directly obtain: 
\begin{theorem}
\label{thm:blowup}
There is an algorithm that 
\begin{itemize}[topsep=0pt,noitemsep,leftmargin=*]
\item takes as input an instance $(G,D)$ of $\rNNT$ where $G$ has depth $\ell$, $n$ input nodes, $o$ output neurons and at most $\lambda$ hidden neurons per layer,
\item computes the blow-up $G'$ of $G$ and an optimal solution to $\rNNT$ on $(G',D)$, and
\item runs in time $|D|^{\bigoh\left( n\cdot \tau \cdot 2^\tau \right)}$, where $\tau=(\ell + \lambda^{\ell+1})\cdot o$.
 \end{itemize}
        \end{theorem}
Observe that the running time in Theorem~\ref{thm:blowup} is polynomial for every architecture of constant size and essentially tight with respect to the lower bound of \citet{FroeseHN22} for architectures with a single hidden neuron. In combination with Proposition \ref{prop:blow-up-opt}, this implies a polynomial-time algorithm which transforms an arbitrary constant-size architecture into a new one which has already been solved to optimality, with no sacrifice in either accuracy or depth.

We remark that the procedure described in this section can be seen as a rough analogue to \emph{improper learning} in learning theory in that it involves a change of the hypothesis class.
 In particular, while we prove that our blow-up procedure does not increase the training error, it might increase the generalization error of empirical risk minimization in learning settings.
We also point out that a similar procedure was described in an earlier work of Maass, albeit in it was applied in a different context~\cite{Maass97}.

\section{Training Linear Networks}
\label{sec:lin}

\smallskip
\noindent \textbf{Untanglings of Linear Networks.}\quad
As discussed in Section~\ref{sec:known}, the complexity of \lNNT\ is open when restricted to architectures with arbitrarily many input nodes while considering the number of other nodes in the network to be fixed by a constant, and very few classes of architectures are known to admit polynomial-time training for linear activations. Here, we identify a general substructure of the architecture which, when present, guarantees the polynomial-time solvability of \lNNT.
         
Specifically, an \emph{untangling} of an architecture $G=(V,E)$ with input nodes $x_1,...,x_\alpha$ and output nodes $y_1,...,y_\beta$ is a partitioning $\Pi=(B,C_1,...C_\alpha,D_1,...,D_\beta)$ such that:
\begin{itemize}[noitemsep, topsep=0pt,leftmargin=*]
\item Each $G[C_i]$, $1\leq i\leq \alpha$ contains $x_i$ as its only source;
\item Each $G[D_j]$, $1\leq j \leq \beta$ contains $y_j$ as its only sink; and
\item For each $x_i$ and $y_j$, if $y_j$ is reachable from $x_i$ in $G$, there exist $s\in C_i$ and $t\in D_j$ such that $st \in E$.
\end{itemize}

As the first result in this section, we will show that if one is given a network architecture along with an untangling, it is possible to solve \lNNT\ in polynomial time. Towards this, we observe that since all transformations in such networks are affine, it is possible to assume without loss of generality that non-zero biases are only present on output neurons.
Indeed, one can slide the biases along the edges in the network towards the output neurons, and if we multiply them by the weights of edges met along the way, the computed function does not change. As an example, given an edge with weight $a_1$ from a neuron associated with bias $b_2\neq 0$ to a neuron associated with bias $b_1$, we can alter the biases by setting $b_2:=0$ and $b_1:=b_1+ a_1\cdot b_2$ since $a_1 \cdot ( x + b_2) +b_1 = a_1\cdot x + (b_1+ a_1\cdot b_2)$. Applying this iteratively along all the edges of a network yields:

\begin{observation}
\label{obs:nobias}
For each neural network $\mathcal{N}_{a,b}$ with linear activation functions, there exists a neural network $\mathcal{N}_{a,b'}$ such that $b'$ is non-zero only on the output neurons and $f_{\mathcal{N}_{a,b}}=f_{\mathcal{N}_{a,b'}}$.
\end{observation}

Observation~\ref{obs:nobias} allows us to proceed towards establishing the claimed tractability result.

\iflong
\begin{theorem}
\fi
\ifshort
\begin{theorem}
\fi
\label{thm:untang}
\lNNT can be solved in time $\poly(|D|+ |V|)$ if an untangling of the architecture $G=(V,E)$ is provided as part of the input.
  \end{theorem}
      
\iflong
\begin{proof}
Consider an instance $(G, D)$ along with an untangling $\Pi$ of $G$. 
For each pair $(x_i,y_j)$ such that $y_j$ is reachable from $x_i$ in $G$, we choose one edge $e_{ij}=(s,t)$ such that $s\in C_i$ and $t\in D_j$. Let $E_{CD}$ be the set containing all edges which are selected in this manner.
For each $C_i$, we select an edge set $E_{C_i}\subseteq E$ which forms an out-branching spanning $C_i$, i.e., $E_{C_i}$ forms a spanning tree of $C_i$ with $n_i$ as its only source. Let $E_C=\cup_i E_{C_i}$.
For each $D_j$, we analogously select an in-branching $E_{D_j}\subseteq E$ spanning $D_j$, i.e., $E_{D_j}$ forms a spanning tree of $D_j$ with $o_j$ as its only sink, and similarly $E_D=\cup_j E_{D_j}$.
 
Now, we set the weights in $G$ as follows. For each edge in $E_{C}\cup E_D$, we set $a_e:=1$, while for each edge $e'\not \in E_{CD} \cup E_C \cup E_D$ we set $a_{e'}:=0$. 
 The only edges whose weights are not yet set are the elements of $E_{CD}$, and moreover it is necessary to also set the biases for the output neurons (see Observation~\ref{obs:nobias}). 

Let $P_{ij}$ be the sum of products of edge weights along all paths from input $i$ to output $j$, and  $b_j$ the bias at the $j$-th output. Note that in a linearly activated neural network, the $j$-th output value is given by $\sum_i P_{ij}x_i + b_j$. Note that $P_{ij} = 0$ if there is no such path. Consider then the quadratic optimization instance over $n \times d$-dimensional variable matrix $\mathbf{T} = (t_{ij})_{ij}$ and the $d$-dimensional variable vector $\mathbf{z} = (z_k)_k$. We constrain the variables $t_{ij}$ by  requiring each $t_{ij}=0$ whenever there is no path from input $i$ to output $j$ in $G$.

Let $\lambda$ be the value of the optimal solution $\mathbf{T}^\ast, \mathbf{z}^\ast$ with respect to the optimization function $\sum_{(x,y) \in D} || y - \mathbf{T} x - \mathbf{z}^T||_2^2$, and let $\lambda'$ be the optimal value for the instance of $\lNNT$ according to Eq. \eqref{eq:loss} achieved by an optimal setting of weights and biases. Since the optimal sum-of-products matrix $\mathbf{P}^\ast$ and biases $\mathbf{b}^\ast$ corresponding to the optimal weight and bias setting is a particular solution for the quadratic optimization problem outlined above, and \eqref{eq:loss} specializes to the expression $\sum_{(x,y) \in D} || y - \mathbf{P}^\ast x - (\mathbf{b}^\ast)^T||_2^2$ we have that $\lambda \leq \lambda'$.
Conversely, the setting of weights $a$ and biases $b$ achieved by $a_{e_{ij}} = t^\ast_{ij}$ for $e_{ij}$ as above and $b_j = z^\ast_j$ produces a corresponding sums-of-product matrix $\mathbf{P}$ and biases $\mathbb{b}$ such that $\mathbf{P} = \mathbf{T}^\ast$ and $\mathbf{b} = \mathbf{z}^\ast$ holds, and $\lambda = \lambda'$ follows. This proves that the computed weight setting is optimal given the network architecture $G$. The quadratic programming problem of computing the optimal solutions $\mathbf{T}^\ast, \mathbf{z}^\ast$ can once again be solved in polynomial time via, e.g., the ellipsoid method \citep{Kozlov79}. 
\end{proof}
\fi
 
 \smallskip
 \noindent \textbf{Computing Untanglings.}\quad
Given Theorem~\ref{thm:untang}, the natural next question concerns the complexity of actually computing an untangling in a neural network. The aim of this section is to provide a comprehensive answer to that question. As our first result, we prove that deciding whether a given architecture has an untangling is, in general, intractable.

\iflong
\begin{theorem}
\fi
\ifshort
\begin{theorem}
\fi
\label{thm:nph}
Deciding whether a given architecture $G$ has an untangling is \NP-hard.
\end{theorem}
\iflong
\begin{proof}
\fi
\ifshort
\begin{proof}[Proof Sketch]
\fi
We provide a polynomial reduction from the \NP-hard \textsc{Dominating Set 3-Partitioning} problem~\citep{HeggernesT98}: Given an undirected graph $H = (V, E)$, decide whether its vertex set can be partitioned as  $V = A_1 \cup A_2 \cup A_3$ such that $A_1,A_2$ and $A_3$ are all dominating sets of $H$. Here, a \emph{dominating set} of a graph is a subset $X$ of its vertices such that every vertex $v$ is either in $X$ itself, or adjacent to some $w \in X$; we say that such $w$ \emph{dominates} $v$. 

Let $H = (V,E)$ be an undirected graph. We first describe how to transform $H$ into an architecture $G$ such that $G$ has an untangling if and only if $H$ can be partitioned into three dominating sets. $G$ is constructed as a three-layered architecture, that is, with vertex set $L_1 \cup L_2 \cup L_3$, where $L_1 = \{d_1,d_2,d_3\}$, $L_2 = V$, and $L_3$ consists of $|V|+1$ disjoint copies of $V$, that is, $L_3 = \bigcup_{v\in V} \{v^{(1)},\ldots,v^{(|V|+1)}\}$.
Additionally, $G$ has all possible edges from $L_1$ to $L_2$, and $v \in L_2$ has an edge towards $u^{(i)} \in L_3$ in $G$ if and only if $u=v$ or $v$ was adjacent to $u$ in $H$.
\ifshort
To complete the proof, it suffices to establish that every untangling in $G$ implies a partitioning of $H$ into three dominating sets, and vice-versa that every such partitioning of $H$ guarantees that there is an untangling in $G$.
\fi
\iflong

Let us first argue that every partition of $V$ into three dominating sets ensures that $G$ has an untangling: Let $A_1,A_2,A_3$ be three such dominating sets. Recall that we have $L_2 = V$, so it makes sense to think of $A_i \subseteq V$ also as subsets of $L_2$. Therefore, we may set $C_i = \{d_i\} \cup A_i$ for $i=1,2,3$, and choose every $D_j$ to be some singleton $\{v_j\} \subseteq L_3$ (so formally, $j=1,\ldots,|V|^2 + |V|$).
Clearly, there is a path from $d_i$ to every $v \in L_3$, and therefore, there must be an edge between $C_i$ and $D_j = \{v\}$ for every pair $(i,j)$. But since $A_i$ is a dominating set for every $i$, in particular, $v$ must be dominated by some $u \in A_i$, and by construction of $G$, there is an edge $uv$ from $C_i$ to $D_j$.

Towards the other direction, consider some untangling of $G$. Note that, if there is any untangling of $G$, 
then there is also an untangling where $D_j$ does not contain any vertices from $L_2$: Suppose that $v^{(i)}$ is the only sink of $D_j$ for some $j$, and that moreover $u \in L_2 \cap D_j$. Since there are $|V|+1$ copies of $v$ in $L_3$ and the sets in an untangling are disjoint, not all of the respective $D_j$ can contain a vertex from $L_2$. Hence, at least one of the copies of $v$ in $D_j$ is a singleton. By symmetry, this implies that we can safely move all the vertices from $L_2$ in $D_j$ into one of the $C_i$ connected to the singleton $\{v\}$. Therefore, it suffices to argue that an untangling of $G$ with only singletons in the $D_j$ implies that $H$ had a partition into three dominating sets. Note that $C_1,C_2,C_3$ satisfy the following condition: since every output is reachable from every input in $G$, in particular, every $C_i$ must be connected to every singleton $D_j$. Since the adjacencies between $L_2$ and $L_3$ are given by  the domination-relation on $V$, this implies that every $C_i$ must dominate every vertex in $V$. 
\fi
\end{proof}

While negative, the intractability of computing an untangling in a general architecture is far from unexpected. For our next result, we take aim at the restriction of \lNNT where the architectures have an arbitrary but fixed bound on the number of hidden neurons---a case which is notable not only due to the fact that the target architectures will typically be much smaller than the training set, but also because its complexity remains open (cf.\ Section~\ref{sec:known}). We show that under this restriction, the problem of determining whether the architecture has an untangling is solvable in linear time.

\iflong
\begin{theorem}
\fi
\ifshort
\begin{theorem}
\fi
\label{thm:untangbounded}
There is an algorithm which either computes an untangling in a given architecture $G=(V,E)$ with $k$ hidden neurons in time $k^{\bigoh(k)}\cdot |V|$, or correctly determines that none exists.
\end{theorem}

\iflong
\begin{proof}
\fi
\ifshort
\begin{proof}[Proof Sketch]
\fi
Let $W\subseteq V$ be the set of hidden neurons in $G$. We begin by exhaustively branching over all of the at most $k^{\bigoh(k)}$ partitionings of $W$ into sets $\Pi'=(B,C'_1,...C'_\alpha,D'_1,...,D'_\beta)$ (for all choices of $\alpha$ and $\beta$, both of which are trivially upper-bounded by $k$). For each such choice of $\Pi'$, the algorithm then proceeds by determining whether it can be extended to an untangling of $G$ by adding an input node to each $C'_i$, $1\leq i\leq \alpha$, an output neuron to each $D'_j$, $1\leq j\leq \beta$, and keeping all remaining input nodes and output neurons as singletons in the untangling, in a way which satisfies the reachability constraint of an untangling. To this end, we will employ a further set of exhaustive branching rules which will iteratively assign singleton output neurons and input nodes to the sets in $\Pi'$. 
\ifshort
Intuitively, whenever we identify that an input node is missing a connection to an output node (due to at least one of these not being assigned to a part in $\Pi'$), we exhaustively branch over all ways one of these two neurons can be assigned to a part in $\Pi'$. 
\fi
\iflong 
Whenever we attempt to extend a set in $\Pi'$ in this way, we also check that the newly added node is indeed the sole source or sink, respectively, in that set; if this sanity check fails, we immediately discard the given choice and proceed to the next branch.

Given a choice of $\Pi'$ fixed in some branch of the algorithm, we say that an input node $v$ is \emph{satisfied} if every output neuron $w$ it can reach in $G$ is assigned to a part which can be reached by an edge from the part $v$ is assigned to; in particular, if $v$ or $w$ have not been assigned to a part in $\Pi'$ yet, we consider them to be assigned to a singleton part.  Crucially, if at any point we reach a situation where every input node is satisfied and all parts have been assigned to a node, then the extension of $\Pi'$ computed in the branching procedure is already an untangling which the algorithm can output. 

For our second branching step, let us consider an input node $v$ which is not yet satisfied due to the existence of an output neuron $w$ such that there is no edge between the part of $v$ and the part containing $w$. In this case, the only way the current partitioning $\Pi'$ could be extended to an untangling is if either (1) $v$ is assigned to some part $C'_i$, $1\leq i\leq \alpha$ (assuming that $v$ has not been assigned to some part in one of the preceding branching steps), or (2) if $w$ is assigned to some part $D'_j$, $1\leq j\leq \beta$ (assuming that $w$ has not been assigned to a part in one of the preceding branching steps). Since the total number of parts is upper-bounded by $\alpha+\beta\leq k$, we can once again exhaustively branch over the at most $k$ many options of assigning $v$ or $w$ to a free part in $\Pi'$; observe that if both $v$ and $w$ are already assigned to parts, then this leads to the currench branch being aborted. We repeat this procedure until all input nodes are satisfied, or until there are no free parts in $\Pi'$ remaining, or until the aforementioned sanity check fails. 

To argue the claimed running time, we note that by first indexing which output neurons are reachable from each of the at most $k$ hidden neurons in the first layer in total time at most $\bigoh(k^2|V|)$, we can determine for each input node $v$ the set of reachable output neurons in time $\bigoh(k)$ as a simple union of the indexed reachability sets of the neighbors of $v$. Together with the fact that the total depth of the branching procedure in the second step of the algorithm is upper-bounded by $k$, we conclude that the running time of the algorithm can be upper bounded by $k^{\bigoh(k)}\cdot |V|$.

Finally, to argue correctness, on one hand every extension $\Pi'$ that the algorithm outputs must be an untangling: each part contains precisely one source or sink, and the existence of the required edges between parts follows by the fact that each input node is satisfied. On the other hand, if an untangling $\Pi$ exists then the algorithm will at some point consider a partitioning $\Pi'$ in the first step that forms the projection of $\Pi=(B,C_1,...C_n,D_1,...,D_o)$ on the hidden neurons of $G$ (where $o$ is the number of output neurons). Then at each branching step, it is possible to verify that the algorithm will always have a branch that will allow it to expand a part $C'_i$ or $D'_j$ in $\Pi'$ to make it precisely $C_i$ or $D_j$, respectively. Hence, at least one branch of the algorithm will result in an untangling of $G$, as required.
\fi
\end{proof}

As our final result, in the next Theorem~\ref{thm:untangtw} we show that it is also possible to compute untanglings in linear time even on architectures with many hidden neurons, provided that such architectures are ``well-structured'' in a graph-theoretic sense. In particular, we show that this holds for architectures which are tree-like in the sense of having bounded \emph{treewidth}: a highly established structural graph parameter which measures how tree-like a graph is~\citep{RobertsonS83}. Specifically, here we consider the treewidth of the underlying undirected graph, i.e., the simple graph obtained by replacing each directed edge in the architecture by an undirected one.

To put Theorem~\ref{thm:untangtw} into context, we remark that the proof is based on an application of Courcelle's well-known Theorem~\citep{Courcelle90} and hence a definition of treewidth will not be required for an exhibition of the result or its proof; instead, we refer to the multitude of existing surveys and materials on the topic~\citep{CBook,Bodlaender16a}. We remark that the result is incomparable to Theorem~\ref{thm:untangbounded}: while it can easily be shown that every architecture with $k$ hidden neurons also has treewidth at most $k$ and hence Theorem~\ref{thm:untangtw} can be applied to a strict superclass of architectures, the running time guaranteed here has a nonelementary dependency on the treewidth of the graph (as opposed to the $k^{\bigoh(k)}$ dependency in Theorem~\ref{thm:untangbounded}).

\iflong
\begin{theorem}
\fi
\ifshort
\begin{theorem}
\fi
\label{thm:untangtw}
Let $t$ be an arbitrary but fixed integer. Given an architecture $G=(V,E)$ of treewidth at most $t$, it is possible to compute an untangling (or correctly determine that none exists) in time $f(t)\cdot |V|$ for some computable function $f$.
\end{theorem}

\iflong
\begin{proof}
We prove the theorem by constructing a formula $\Phi(X)$ in Monadic Second Order Logic (MSOL) such that $G\models \Phi(S)$ for some $S\subseteq E$ if and only if $G$ admits an untangling $\upsilon=(B,C_1,\dots,C_\alpha,D_1,\dots,D_\beta)$. Once that is done, we invoke Courcelle's Theorem~(\citeyear{Courcelle90}) to find such a set $S$ or determine that none exists; moreover, the set $S$ will allow us to construct such an untangling. 

Before proceeding, we recall the basics of Courcelle's Theorem and MSOL while referring the dedicated materials for a more detailed exposition~\citep{CBook,Bodlaender16a}. A formula in MSOL is built from atomic formulas over variables which are interpreted over vertices, edges, vertex sets and edge sets in the graph. Atomic formulas are connected with the usual Boolean connectives $(\lnot,\land,\lor,\rightarrow,\leftrightarrow)$, and variables may either appear as free or be existentially or universally quantified. The atomic formulas can express the vertex-edge incidence relation, set inclusion, and equality; in the case of directed graphs, they can also express the existence of a directed edge from one vertex to another. Courcelle's Theorem then states that, for each fixed MSOL formula $\Psi(X)$ and fixed integer $t$, it is possible to determine whether a given graph $G'=(V',E')$ admits an interpretation $S$ such that $G\models \Psi(S)$ (and also output such $S$ if it exists) in time linear in $|V'|$.
	
To describe the formula $\Phi(X)$, we recall that for each part $C_i$, $1\leq i\leq \alpha$, $G[C_i]$ contains the input node $x_i$ as its sole source. Hence $G[C_i]$ contains a spanning tree with $x_i$ as its only node of in-degree $0$; such a tree is called an \emph{out-branching} of $G[C_i]$, and we denote by $E^C_i$ the edges in this tree. 
Analogously, each subgraph $G[D_j]$, $1\leq j\leq \beta$, contains a spanning tree with output neuron $y_j$ as its only node of out-degree $0$ (a so-called \emph{in-branching} of $G[D_j]$), and we denote by $E^D_j$ the edges in this tree. Given an untangling $\upsilon$, let $E_\upsilon$ be the set of edges obtained as the union of the edges in all of the associated in-branchings and out-branchings, i.e., $E_\upsilon=\big(\bigcup_{1\leq i\leq \alpha}E^C_i \big) \cup \big(\bigcup_{1\leq j\leq \beta}E^D_j \big)$. 

Crucially, it is possible to construct $\upsilon$ from the edge set $E_\upsilon$ in linear time: each $C_i$, $1\leq i\leq \alpha$ is the set of vertices that can be reached from $x_i$ via $E_\upsilon$, and similarly each $D_j$, $1\leq j\leq \beta$, is the set of vertices that can reach $y_j$ via $E_\upsilon$. Our aim will be to construct $\Phi(X)$ so that $G\models \Phi(S)$ if and only if $S=E_\upsilon$ for some untangling $\upsilon$. Once that is done, the theorem will follow from Courcelles's Theorem.

To achieve this, we set 
$$\Phi(X)=\forall Y: \big(\texttt{CC}(Y,X) \implies (\texttt{InB}(Y) \vee \texttt{OutB}(Y))\big) $$ 
$$ \wedge \quad \forall Y,Z: \big(\texttt{CC}(Y,X) \wedge \texttt{InB}(Y) \wedge \texttt{CC}(Z,X)$$ $$\wedge \texttt{OutB}(Z,X)\big) \implies \big(\texttt{Path}(Y,Z)\implies \texttt{Arc}(Y,Z)),$$
where:
\begin{itemize}
\item $\texttt{CC}(Y,X)$ is a subformula which is satisfied if and only if the edge set $Y$ forms a connected component of the edge set $X$;
\item $\texttt{InB}(Y)$ is a subformula which is satisfied if and only if the edge set $Y$ is an in-branching whose source is also a source in $G$;
\item $\texttt{OutB}(Y)$ is a subformula which is satisfied if and only if the edge set $Y$ is an out-branching whose sink is also a sink in $G$;
\item $\texttt{Path}(Y,Z)$ is a subformula which is satisfied if and only if there is a directed path in $G$ from a source in $Y$ to a sink in $Z$;
\item $\texttt{Arc}(Y,Z)$ is a subformula which is satisfied if and only if there is an edge from a vertex incident to an edge in $Y$ to a vertex incident to an edge in $Z$.
\end{itemize}
It is easy to see that each of the subformulas described above can be straightforwardly expressed in Monadic Second Order logic. Since $\Phi(X)$ is only satisfied if (1) each connected component of the edge set $X$ forms either an out-branching or in-branching, and (2) reachability from an input node to an output neuron is preserved by the edges between these in-branchings and out-branchings, we conclude that indeed $G\models \Phi(S)$ if and only if $S=E_\upsilon$ for some untangling $\upsilon$, as desired.
\end{proof}
\fi

\section{Concluding Remarks}
We have examined the theoretical boundaries of computational tractability for training neural networks both for the fundamental cases of ReLU and linear activations, complementing a flurry of recent results establishing increasingly stronger lower bounds for network training. Our results generalize recent algorithms for optimal ReLU-activated network training~\citep{AroraBMM18,BoobDL22} and are among the first to identify a non-trivial class of linear-activated architectures admitting polynomial-time training. 

A natural direction for future work is to examine the extent to which the identified islands of tractability can be generalized before stepping into intractable territory. A long-term goal is to obtain a clear cutoff between intractable and tractable cases of network training. We conclude with three specific questions directly arising from our results:
\begin{enumerate}
\item Is \lNNT polynomial-time tractable when restricted to constant-treewidth architectures? 
\item Is it possible to solve \rNNT\ for constant-size architectures containing hidden neurons with out-degree greater than $1$?
\item Can the obtained results be generalized to other loss functions?
\end{enumerate}

\begin{ack}
The authors confirm that there are no competing interests.
We acknowledge support by the Austrian Science Fund (FWF, START Project Y1329).
Mathis Rocton further acknowledges support by the \includegraphics[width=0.5cm]{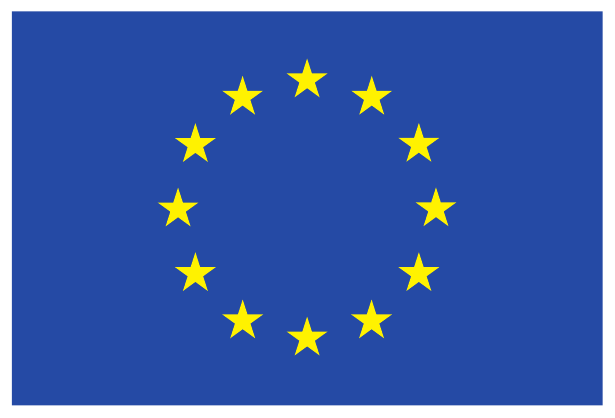} European Union's Horizon 2020 research and innovation COFUND programme (LogiCS@TUWien, grant agreement No 101034440).
\end{ack}

  \bibliographystyle{icml2023}
     \bibliography{ref}
\end{document}